\documentclass[sigconf]{aamas}
\renewcommand\footnotetextcopyrightpermission[1]{}
\usepackage{balance} 

\acmSubmissionID{991}

\title[Formally-Sharp DAgger for MCTS]{Formally-Sharp DAgger for MCTS: Lower-Latency Monte Carlo Tree Search using Data Aggregation with Formal Methods}

\author{Debraj Chakraborty}
\affiliation{%
	\institution{Universit\'{e} Libre de Bruxelles}
	\city{Brussels}
	\country{Belgium}
}
\email{debraj.chakraborty@ulb.be}
\orcid{0000-0003-0978-4457}

\author{Damien Busatto-Gaston}
\affiliation{
  \institution{Univ. Paris Est Créteil, LACL, F-94010}
  \city{Creteil}
  \country{France}}
\email{damien.busatto-gaston@u-pec.fr}
\orcid{0000-0002-7266-0927}

\author{Jean-Fran\c{c}ois Raskin}
\affiliation{%
	\institution{Universit\'{e} Libre de Bruxelles}
	\city{Brussels}
	\country{Belgium}
}
\email{jean-francois.raskin@ulb.be}
\orcid{0000-0002-3673-1097}

\author{Guillermo A. P\'erez}
\affiliation{%
	\institution{University of Antwerp}
	\city{Antwerp}
	\country{Belgium}
}
\email{guillermo.perez@uantwerpen.be}
\orcid{0000-0002-1200-4952}

\begin{abstract}
We study how to efficiently combine formal methods, Monte Carlo Tree Search (MCTS), and deep learning in order to produce high-quality receding horizon policies in large Markov Decision processes (MDPs). 
In particular, we use model-checking techniques to guide the MCTS algorithm in order to generate offline samples of high-quality decisions on a representative set of states of the MDP.
Those samples can then be used to train
a neural network that imitates the policy used to generate them. 
This neural network can either 
be used as a guide on a lower-latency MCTS online search,
or alternatively be used as a full-fledged policy when minimal latency is required.
We use statistical model checking to detect when additional samples are needed and to focus those additional samples on configurations where 
the learnt neural network policy differs from the (computationally-expensive) offline policy.
We illustrate the use of our method on MDPs that model the 
Frozen Lake and Pac-Man environments --- two popular benchmarks to evaluate reinforcement-learning algorithms.
\end{abstract}


\usepackage{amsfonts,amsmath}
\usepackage{url}
\usepackage{hyperref}
\usepackage[obeyDraft]{todonotes}
\usepackage[ruled,noend,linesnumbered]{algorithm2e} 
\usepackage{booktabs}
\usepackage{pgfplots}

\usepackage{xspace} 
\usepackage{xparse} 

\newcommand\newmath[2]{\newcommand#1{\ensuremath{#2}\xspace}}
\newcommand\renewmath[2]{\renewcommand#1{\ensuremath{#2}\xspace}}

\newcommand\newmathope[2]{\newcommand#1{\ensuremath{\operatornamewithlimits{#2}}\xspace}}

\newmath{\N}{\mathbb{N}}
\newmath{\Z}{\mathbb{Z}}
\newmath{\Q}{\mathbb{Q}}
\newmath{\R}{\mathbb{R}}

\newmathope{\argmin}{\arg\min}
\newmathope{\argmax}{\arg\max}

\renewmath{\Pr}{\mathbb P}
\newmath{\Dist}{\mathcal D}
\newmath{\Supp}{\mathsf{Supp}}
\newmath{\E}{\mathbb E}

\newmath{\Reward}{\mathsf{Reward}}
\newmath{\AReward}{\mathsf{AReward}}

\newmath{\FPaths}{\mathsf{Paths}}
\newmath{\first}{\mathsf{first}}
\newmath{\second}{\mathsf{second}}
\newmath{\last}{\mathsf{last}}
\newmath{\States}{\mathsf{States}}
\newmath{\Val}{\mathsf{Val}}
\newmath{\Win}{\mathsf{Win}}
\newmath{\ValSafety}{\mathsf{ValSafety}}
\newmath{\ValReach}{\mathsf{ValReach}}
\newmath{\Shield}{\mathsf{Shield}}
\newmath{\len}{\mathsf{len}}

\newmath{\Opt}{\mathsf{Opt}}

\newmath{\NN}{\mathsf{NN}}

\newmath{\reward}{\mathsf{reward}}
\newmath{\numsamples}{\mathsf{count}} 
\newmath{\children}{\mathsf{children}}
\newmath{\mctsvalue}{\mathsf{approxValue}}
\newmath{\total}{\mathsf{total}}
\newmath{\I}{\mathcal{I}}
\newmath{\iter}{\mathsf{iter}}

\newmath{\A}{\mathscr A}
\newmath{\Apsi}{\A^{\psi}}
\newmath{\Aphi}{\A^{\varphi}}

\begin{document}


\pagestyle{fancy}
\fancyhead{}


\maketitle


\section{Introduction}

Markov decision processes (MDPs) are frameworks to
 model sequential decision making. They are discrete-time stochastic models where an agent chooses actions based on the current state. The agent then receives a reward and the state of the MDP is updated based on a probabilistic transition function.
Exact algorithms, or formal methods, for MDPs have been studied since the 1950s and efficient (and symbolic) versions of these algorithms have been implemented in probabilistic model-checking tools such as PRISM~\cite{KNP11} and Storm~\cite{STORM22}. The latter, as well as other tools, are regularly compared with respect to a large body of benchmarks from the Quantitative Verification Benchmark Set~\cite{qcomp2020}.

While tools like PRISM and Storm can handle very large systems, some applications arising from real-world systems and video games like {\sc{Pac-Man}} are still out of reach. In contrast, novel (deep) reinforcement learning (RL) techniques or online heuristic search techniques, like Monte Carlo Tree Search (MCTS)~\cite{DBLP:journals/tciaig/BrownePWLCRTPSC12}, are able to produce policies for larger MDPs~\cite{DBLP:conf/isola/AshokBKS18}, albeit at the cost of either high sample complexity (\textit{i.e.}~they require much data to be trained), or high latency (\textit{i.e.}~they require much time before choosing a next action), and weaker performance guarantees.

In this work, we aim at combining exact  methods, such as model checking, and MCTS to improve the quality of policies synthesized in large MDPs. Concretely, we make use of the MCTS algorithm with \emph{symbolic advice} (coming from formal methods), as proposed in~\cite{DBLP:conf/concur/Busatto-Gaston020}, to increase reliability of MCTS. Further, to improve the latency of MCTS augmented with advice, we propose to replace advice coming from exact algorithms with a neural network, trained on data from the exact advice, that we call \emph{neural advice}.
Finally, we also experiment with training a \emph{surrogate} neural-network policy to imitate MCTS (with advice) altogether. Once more, to realize this efficiently, with respect to sample complexity, we leverage exact methods to obtain ``perfect data'' and we generate additional samples on demand when the performance of the learnt neural network does not match the quality of the policy computed offline. This step uses statistical model checking~\cite{DBLP:reference/mc/2018} instead of classical metrics from machine learning.

\subsubsection*{Contribution.}
 We consider our main contributions to be (1) an \emph{expert imitation framework} to train a neural network in order to replace exact advice by lower-latency neural advice, or to imitate the \emph{expert} policy that can be computed offline, and (2) this imitation framework relies on a data generation algorithm which leverages formal methods to obtain ``perfect data'' for our samples and  
 to
 generate additional samples
 , as long as statistical model checking indicates that it is required to improve the quality of the imitation.

\paragraph{Imitating experts} Imitation can take different forms depending on the expert policy (or advice).
In general, we define a ranking of actions for every state such that the maximally ranked elements are those played by the policy. 
Intuitively, the ranking tells us how good every action is from the current state.
We propose to train a neural network to learn such a ranking function as an offline step. This neural network can then be used as a full-fledged policy or as a neural advice to efficiently guide MCTS. 
The neural advice aims for an expected reward comparable with the expert advice, for a fraction of its online latency.

\paragraph{Data generation and aggregation} Recall that we propose to train a neural network to imitate an expert advice. The expert advice is usually implemented as an exact algorithm. In this case, given a set of inputs for the neural network, the original expert advice can be used to obtain (offline) a ``perfect'' set of corresponding outputs to train on. In contrast, when training a neural network to imitate the full MCTS-with-advice algorithm, 
the data can be noisy for one of two reasons: we are sampling from a randomized policy,
and the expert policy we are imitating may not always match the optimal policy of the MDP.
In both cases, a remaining challenge is to generate a representative set of inputs for the network to be trained on.
We propose to enrich the set of data using formal methods to compare the behaviour of the trained neural network with that of the expert policy in what resembles a Counterexample-Guided Abstraction-Refinement loop~\cite{DBLP:journals/jacm/ClarkeGJLV03}. Our experiments show that this data aggregation loop can speed up learning significantly.

\paragraph{Evaluating a neural network} In order to stop the data aggregation loop, we do not only rely on classical machine learning criteria to evaluate the quality of the generated policies, but also monitor the practical performance of the neural networks.
Indeed, our setting requires taking decisions sequentially for many steps, so that small errors could accumulate over time.
Thus, classical metrics such as computing a loss function on a testing dataset may not be representative of the expected reward a neural network will obtain when used as an advice or a full policy, e.g. a policy may make mistakes at crucial moments despite being almost always correct in its decisions.
Instead, we use statistical model-checking to compute an approximation of the expected reward of our policies.

\subsubsection*{Related work}
Our implementation of the MCTS algorithm with symbolic advice closely follows the approach described in~\cite{DBLP:conf/concur/Busatto-Gaston020}.
However, while they relied on qualitative advice based on quantified Boolean formulae (QBF) and SAT solvers,
we use more quantitative notions instead, based on probabilistic model checking and neural networks.
Our approach also resembles the {\em shielding} framework~\cite{DBLP:conf/aaai/AlshiekhBEKNT18,DBLP:conf/concur/0001KJSB20} used to add safety properties to RL algorithms. One difference is that our technique does not require one to construct the entire MDP, making our work scalable to larger MDPs.

Using deep learning to replace expert (but expensive) policies by learnt policies
is known to be advantageous when the expert policy is unable to meet real-time (latency) constraints (see, e.g.~\cite[Section 5.2]{ivanov2019verisig} and~\cite{learnmpc}).
In order to obtain a satisfactory dataset to train on, we propose a sharp variant of the
DAgger algorithm, a dataset aggregation technique introduced in~\cite{pmlr-v9-ross10a,pmlr-v15-ross11a}. 
A notable difference is that we propose to use model checkers instead of human experts in order to get better-quality data.
We also identify so-called \emph{counterexample} configurations in order to guide the aggregation loop to the most interesting states.
This is reminiscent of counter-example guided abstraction refinement (CEGAR) approaches for hybrid systems such as~\cite{DBLP:conf/aips/ClaviereDS19} that identify states violating a property then focus the deep learning procedures on such states.  

Finally, we rely on statistical model checking~\cite{younes2002probabilistic} to efficiently evaluate particular policies for the system. This consists in running simulated executions of the MDP and computing statistics with confidence guarantees. However, such techniques are not known to find (or approximate) the optimal policies for our reward structures, as that would require using MCTS-like simulation techniques.

\section{Preliminaries}
A \emph{probability distribution} on a countable set $S$ is a function $d:S\to [0,1]$ such that $\sum_{s\in S}d(s)=1$.
We denote the set of all probability distributions on set $S$ by $\Dist(S)$. The support of a distribution $d\in \Dist(S)$ is $\Supp(d)=\{s\in S\mid d(s)>0\}$.

\subsection{Markov chain}
\begin{definition}[Markov chain]\label{def:mc}
	A (discrete-time) Markov chain (MC) is a tuple $M=(S,P,AP,L)$, where
	$S$ is a countable set of states,
	$P$ is a mapping from $S$ to $\Dist(S)$, 
	$AP$ is a finite set of atomic proposition and
	$L$ is the labelling function from $S$ to $2^{AP}$.
\end{definition}
For states $s,s'\in S$, $P(s)(s')$ denotes the probability of moving from state $s$ to state $s'$ in a single transition and we denote this probability $P(s)(s')$ as $P(s,s')$.
We say that the atomic proposition $a$ holds in
a state $s$ if $a\in L(s)$.
For a Markov chain $M$, a
\emph{finite path} $p = s_0s_1\ldots s_i$ of length $i\ge 0$ is
a sequence of $i+1$ consecutive states such that for all $t\in[0,i-1]$, $s_{t+1}\in \Supp(P(s_t))$.
Similarly, An infinite path is an infinite sequence $p = s_0s_1s_2\ldots$ of states such that for all $t\in\N$, $s_{t+1}\in \Supp(P(s_t))$.
For a finite or infinite path $p=s_0s_1\ldots$, we denote its $(i+1)^{th}$ state by $p[i] = s_i$.
Let $p=s_0s_1\ldots s_i$ and $p'=s'_0s'_1\ldots s'_j$ be two paths
such that $s_i=s'_0$.
Then, $p\cdot p'$
denotes $s_0s_1\ldots s_is'_1\ldots s'_j$.
For an MC $M$, the set of all finite paths of length $i$ (resp.~infinite paths)
is denoted by $\FPaths^i_{M}$ (resp.~$\FPaths^{\omega}_{M}$).
We denote the set of all finite paths in $M$ by $\FPaths_{M}$ and the set of finite paths of length
at most $H$ by $\FPaths^{\leq H}_{M}$.
For $p\in\FPaths_{M}$, let $\FPaths^{\omega}_{M}(p)$ denote the set of paths $p'$ in $\FPaths^{\omega}_{M}$ such that there exists $p''\in\FPaths^{\omega}_M$ with $p'=p\cdot p''$. $\FPaths^{\omega}_{M}(p)$ is called the cylinder set of $p$.

The $\sigma$-algebra associated with the MC $M$ 
is the smallest $\sigma$-algebra that contains the cylinder sets $\FPaths_{M}^{\omega}(p)$ for all $p\in \FPaths_{M}$. For a state $s$ in $S$, a measure is defined for the cylinder sets as
\begin{align*}
	\Pr_{M,s}(\FPaths_{M}^{\omega}(s_0s_1\ldots s_i))=&\begin{cases}
		\prod_{t=0}^{i-1}P(s_t)(s_{t+1}) &\text{if } s_0 = s\\
		0 &\text{otherwise.}
	\end{cases}
\end{align*}
We also have $\Pr_{M,s}(\FPaths_{M}^{\omega}(s)) = 1$ and $\Pr_{M,s}(\FPaths_{M}^{\omega}(s')) = 0$ for $s'\neq s$.
Using Carath\'{e}odory's extension theorem \cite[section 1.3.10]{AshDole99}, this can be extended to a unique probability measure $\Pr_{M,s}$ on the aforementioned $\sigma$-algebra. In particular, if $\mathcal C\subseteq\FPaths_{M}$ is a set of finite paths forming pairwise disjoint cylinder sets, then
$\Pr_{M,s}(\cup_{p\in \mathcal C}\FPaths_{M}^{\omega}(p))=
\sum_{p\in \mathcal C}\Pr_{M,s}(\FPaths_{M}^{\omega}(p))$.
Moreover, if $\Pi\in\FPaths^{\omega}_{M}$ is the complement of a measurable set $\Pi'$, then $\Pr_{M,s}(\Pi)=1-\Pr_{M,s}(\Pi')$.

\subsection{Probabilistic computation tree logic}\label{sec:pctl}
Probabilistic computation tree logic or \emph{PCTL} is a branching temporal logic which formulates conditions on a Markov chain.
PCTL state formulae over a set of atomic propositions $AP$ are defined according the following grammar:
$$\varPhi:= true \mid a \mid \varPhi_1\wedge\varPhi_2 \mid \neg\varPhi \mid \Pr_J(\varphi)$$
where $a\in AP$, $\varPhi_1$ and $\varPhi_2$ are state formulae, $\varphi$ is a path formula and $J\subseteq[0,1]$ is an interval with rational bounds.
PCTL path formulae are defined according the following grammar:
$$\varphi:=\bigcirc \varPhi \mid \varPhi_1\mathcal{U}\varPhi_2 \mid \varPhi_1\mathcal{U}^{\leq n}\varPhi_2$$
where $\varPhi_1$ and $\varPhi_2$ are state formulae and $n\in \N$.

The satisfaction relation $\models$ between an infinite path $p = s_0s_1\ldots$ and a PCTL path formula is defined as follows:
\begin{itemize}
	\item $p\models \bigcirc \varPhi$ if $p[1]\models \varPhi$.
	\item $p\models \varPhi_1\mathcal{U} \varPhi_2$ if $\exists i\in\N$ s.t. $s_i \models \varPhi_2$ and $\forall j<i$, $p[j] \models \varPhi_1$.
	\item $p\models \varPhi_1\mathcal{U}^{\leq n} \varPhi_2$ if $\exists i\leq n$ s.t. $s_i \models \varPhi_2$ and $\forall j<i$, $p[j] \models \varPhi_1$.
\end{itemize}

We define the probability of a 
path formula $\varphi$ holding at 
$s\in S$ by
$$\Pr_M(s\models \varphi) =
\Pr_{M,s}(\{p\in \FPaths^{\omega}_M(s)\mid p\models\varphi\})$$

The satisfaction relation $\models$ between a state $s\in S$ and a PCTL state formula is defined inductively: $s\models true$, $s\models a$ if $a\in L(s)$, and
\begin{itemize}
	\item $s\models \varPhi_1\wedge\varPhi_2$ if $s\models \varPhi_1$ and $s\models \varPhi_2$.
	\item $s\models \neg \varPhi$ if $s\not\models \varPhi$.
	\item $s\models \Pr_J(\varphi)$ if $\Pr_M(s\models \varphi)\in J$.
\end{itemize}

Using the Boolean connectives $\wedge$ and $\neg$, we can define other Boolean connectives such as $\vee$, $\rightarrow$, $\leftrightarrow$.
The $\mathcal{U}$ operator (and its bounded version) also allows us to define other useful operators such as  $\Diamond$ that expresses reachablility or  $\Box$ that expresses safety: \begin{align*}
	\Diamond \varPhi = true\ \mathcal{U} \varPhi&\text{\quad and \quad}
	\Box \varPhi = \neg\Diamond\neg\varPhi\\
	\Diamond^{\leq n} \varPhi = true\ \mathcal{U}^{\leq n} \varPhi&\text{\quad and \quad}
	\Box^{\leq n} \varPhi = \neg\Diamond^{\leq n}\neg\varPhi
\end{align*}

\subsection{Markov decision process}\label{sec:MDP}

\begin{definition}[Markov decision process]\label{def:mdp}
	A Markov decision process (MDP) is a tuple $M=(S,A,P,R,R_T,AP,L)$, where
	$S$ and $A$ are finite sets of states and actions, respectively, 
	$A$ is a finite set of actions,
	$P$ is a mapping from $S\times A$ to $\Dist(S)$, 
	$R$ is a mapping from $S\times A$ to $\R$, 
	$R_T$ is a mapping from $S$ to $\R$, 
	$AP$ is a finite set of atomic proposition and
	$L$ is the labelling function from $S$ to $2^{AP}.$
\end{definition}
$P(s,a)(s')$ denotes the probability that action $a$ in state $s$ leads to state $s'$ and we denote this probability $P(s,a)(s')$ as $P(s,a,s')$. $R(s,a)$ defines the reward obtained for taking action $a$ from state $s$ and $R_T$ assigns a terminal reward to each state in $S$.

\begin{example}[Frozen Lake]\label{ex:frozenLake}
	We can represent the game Frozen Lake~\cite{frozenLakeGym} as an MDP. In this game, a robot moves in a slippery grid. It has to reach the target while avoiding holes in the grid. Each state in the MDP represents the current position of the robot in the grid. The states representing the target and the holes can be assumed to be sink states, i.e., the robot cannot move to any other positions from this state. Part of the grid contains walls and the robot cannot move into it. The frozen surface of the lake being slippery,
	when the robot tries to move by picking a cardinal direction, the next state is determined randomly over the four neighbouring positions of the robot, according to the following distribution weights: the intended direction gets a weight of $10$, and other directions that are not a wall and not the reverse direction of the intended one get a weight of $1$, the distribution is then normalized so that weights sum up to $1$.
	There are no rewards, and the terminal reward is $1$ when the robot reaches the target and $0$ otherwise.
\begin{figure}[h]
	\centering
	\includegraphics[width=0.5\columnwidth]{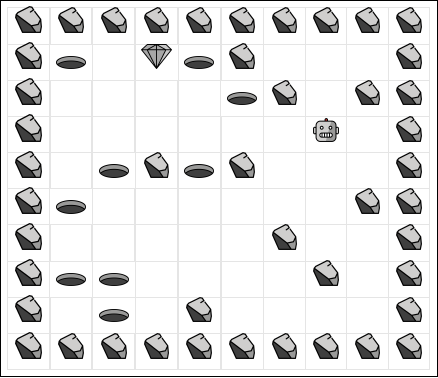}
	\caption{A $10\times 10$ layout for Frozen-Lake}\label{fig:FrozenLake}
\end{figure}

\end{example}

\begin{example}[{\sc{Pac-Man}}]\label{ex:pacman}
	We can represent the multiagent game {\sc{Pac-Man}} as a Markov decision process. In this game Pac-Man  has to eat food pills in an enclosed grid as fast as possible while avoiding the ghosts. The agents (Pac-Man and the ghosts) can travel in the four cardinal directions unless they are blocked by the walls in the grid. Moreover, the ghosts cannot reverse their direction of travel, and are moving uniformly at random among the directions that are left. In the MDP, the states encode a position for each agent\footnote{The last action played by ghosts should be stored as well, as they are not able to reverse their direction.} and for
	the food pills in the grid, while the actions encode individual Pac-Man moves, and while the next state 
	is chosen according to the probabilistic models of the ghosts.
	The reward decreases by 1 at each step, and increases by 10 whenever Pac-Man eats a
	food pill. A win (when the Pac-Man eats all the food pills in the grid) increases the reward by
	500. Similarly, a loss (when the Pac-Man makes contact with a ghost), decreases the reward
	by 500.
	\begin{figure}[h]
		\centering
		\includegraphics[width=\columnwidth]{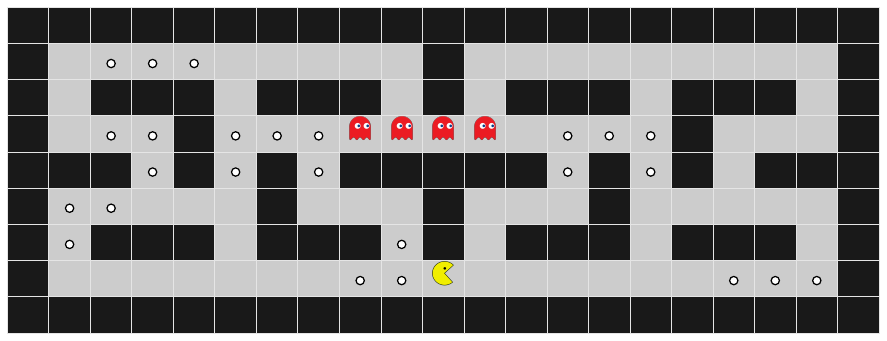}
		\caption{A $21\times 9$ layout for Pac-Man}\label{fig:Pacman}
	\end{figure}
\end{example}

The definitions and notations used for paths in Markov chain can be extended to the case of MDPs. In an MDP, a \emph{path} is
a sequence of states and actions.

For a Markov decision process $M$, a (probabilistic) \emph{policy} is a function $\sigma : \FPaths_{M} \to \Dist(A)$
that maps a path $p$ to a probability distribution in $\Dist(A)$.
A policy $\sigma$ is \emph{deterministic} if the support of
the probability distributions $\sigma(p)$ has size $1$.
A policy $\sigma$ is \emph{memoryless} if $\sigma(p)$ depends only on $\last(p)$,
i.e. if $\sigma$ satisfies that for all $p,p'\in\FPaths_{M}$, $\last(p)=\last(p')\Rightarrow\sigma(p)=\sigma(p')$.

An MDP $M$ and a policy $\sigma$ define an MC $M_{\sigma}$. Intuitively, this is obtained by unfolding $M$, using the policy $\sigma$ and the probabilities in $M$ to define the transition probabilities and ignoring the rewards. Formally $M_{\sigma} = (\FPaths_{M},P_{\sigma},AP,L_{\sigma})$ where for all paths $p\in \FPaths_{M}$, $P_{\sigma}(p)(p\cdot as) = \sigma(p)(a)\cdot P(\last(p),a)(s)$ and $L_{\sigma}(p)=L(\last(p))$.
Thus a finite path $p$ in $\FPaths_{M}(\sigma)$ uniquely \emph{matches} a finite path $p'$ in $M_{\sigma}$ when $\last(p') = p$. This way when a policy $\sigma$ and a state $s$ is fixed, the probability measure $\Pr_{M_{\sigma},s}$ defined in $M_{\sigma}$ is also extended for paths in $\FPaths_M(\sigma)$.
For ease of notation, we write $\Pr_{M_{\sigma},s}$ as $\Pr_s^{\sigma}$. We write the expected value of a random variable $X$ with respect to the probability distribution $\Pr_s^{\sigma}$ as $\E_s^{\sigma}(X)$.

Our goal is to maximize the expected rewards obtained by a policy.
Classically, this can mean maximizing the sum of rewards up to a finite horizon, or maximizing infinite-horizon metrics such as average reward or discounted sum.
In our experiments on Frozen Lake and {\sc{Pac-Man}}, we optimize for the total reward objective after fixing a horizon at which the game ends in a draw. 
\begin{definition}[Total reward]
	The total reward of horizon $h$ for a path $p = s_0a_0\ldots $ in $M$
	is defined as $\Reward^h_{M}(p)=\sum_{i=0}^{h-1}R(s_i,a_i)+R_T(s_h)$.
	The \emph{expected total reward} of a policy $\sigma$ in an MDP $M$,
	starting from state $s$ and for a finite horizon $h\in \N$, is defined as
	$$\Val^h_{M}(s,\sigma) = \E^{\sigma}_s\left[\Reward^h_{M}\right]\,.$$
	The optimal expected total reward of horizon $h$, starting from $s$,
	over all policies $\sigma$ in the MDP $M$ is
	$\Val^h_{M}(s)=\sup_{\sigma}\Val^h_{M}(s,\sigma)$.
\end{definition}

One can show that there is a deterministic policy that achieves this supremum~\cite[Theorem 4.4.1.b]{DBLP:books/wi/Puterman94}.
Such optimal policies may not be memoryless, as one can change their behaviour as the horizon $h$ approaches for example.
As the choice of $h$ is arbitrary, we would like to 
find policies that achieve a good expected total reward 
independently of $h$ (\textit{i.e.}~for every $h$ that is big enough). 
We will focus our search on (randomized) memoryless policies as a result.\footnote{Note that in some situations, randomization can help a memoryless policy emulate the behaviour of a non-memoryless policy~\cite{chatterjee2004trading}.}

\section{Expert policies}\label{sec:planning}

We describe different policies computed using a combination of formal methods, heuristic search algorithms and machine learning that all aim for the optimal expected total reward.

\subsection{Formal methods}\label{subsec:MC}
\paragraph{Model checking} Exact methods can be used to compute a policy that reaches the optimal expected total reward, \textit{e.g.} with dynamic programming (value iteration)~\cite[Section 4.5]{DBLP:books/wi/Puterman94}. They have been efficiently implemented in probabilistic model-checkers such as Storm~\cite{STORM22}, that offer support for a large range of specifications.
More specifically, given a model (Markov chain or MDP), a reward structure and a specification
such as a PCTL formula as defined in Section~\ref{sec:pctl},
Storm can determine whether the input model conforms to the specification and
compute expected rewards for a range of finite or infinite horizon metrics
such as total or average reward. For MDPs, probabilistic model-checkers can also output an optimal policy associated with the optimal expected reward that they compute.
Such tools have been designed with performance in mind and can typically handle models of size up to $10^8$ states. 
Exact methods are thus applicable for smaller MDPs such as the MDP obtained for Frozen Lake in Example~\ref{ex:frozenLake}, but not for larger models such as {\sc{Pac-Man}} (the MDP represented in \figurename~\ref{fig:Pacman} already has approximately $10^{16}$ states). For larger MDPs, formal methods offer alternative techniques \textit{e.g.}~based on sampling.
\paragraph{Symbolic MDP} The MDP can be described symbolically in the PRISM~\cite{KNP11} language,
a guarded command language where one only needs to specify abstract rules that the transitions must satisfy.

\paragraph{Statistical model checking} Computations can make use of statistical model-checking techniques to find good approximations 
of the expected reward of a policy.
By relying on running simulations and computing statistics, it  
offers 
confidence guarantees on the quality of the approximated expected reward.
We also use the Storm model-checker in this context, as it is capable of producing simulated paths efficiently for an MDP in PRISM format.

\paragraph{Scalability} In both cases, one can scale to larger models by focusing on smaller horizons or sub-objectives 
for which the MDP can be abstracted further. The idea is that for simple parts of the specification, the relevant aspects of the model may define a much smaller MDP.
For example, if we focus on a safety objective in {\sc Pac-Man} (not being eaten is a necessity in order to get a good reward), we can ignore the status of the food pellets, \textit{i.e.}~which food has already been eaten, 
reducing the state-space from a size of $10^{16}$ to $10^7$.
Overall, whenever exact methods become too expensive we will rely on heuristic approaches based on a combination of fixed-horizon and sampling-based state-space exploration techniques.

\subsection{Monte Carlo tree search}

We consider online procedures where the controller, upon visiting a new state $s$, computes what action $a$ it thinks is best, and plays it. Then, the state evolves stochastically to a new state $s'$ according to the distribution $P(s,a)$. 
This is known as \emph{decision-time planning}~\cite[Chapter 8.8]{sutton2018reinforcement}.
Specifically, we rely on the 
\emph{receding horizon control} approach, where the controller fixes a small horizon $H$ and finds an action that optimizes the expected total reward of horizon $H$. 
This approach is meant to select decisions with good short-term consequences, while a well-chosen terminal reward function can be used to predict long-term behaviors from there. 

Given an initial state $s$, \emph{Monte Carlo tree search} or MCTS algorithm \cite{DBLP:journals/tciaig/BrownePWLCRTPSC12}
is a popular policy that
incrementally
constructs a search tree rooted at $s$
describing paths of the MDP.
This process goes on until a specified budget
(of number of iterations or time) is exhausted.
An iteration constructs a path by following a decision policy
to \emph{select} a sequence of nodes in the search tree. When a node that is not
part of the current search tree is reached, the tree is expanded with this new node,
whose expected reward is approximated by \emph{simulation}.
This value is then used to update the knowledge of all selected nodes in \emph{backpropagation}.
Thus, we get a value estimation $\mctsvalue(s,a)$ for all actions $a$ from the state $s$. Then the controller takes the action maximizing $\mctsvalue(s,a)$.

\subsection{\textbf{Monte Carlo tree search with advice}}\label{sec:MCTSa}
MCTS can be augmented with \emph{symbolic advice}~\cite{DBLP:conf/concur/Busatto-Gaston020} which prune a part of the search tree according to formal specifications meant to differentiate the ``good'' and ``bad'' parts of the tree.

A qualitative approach considers a logical formula that the ``good'' paths need to satisfy.
For example, consider the set of states labelled with $loss$ where Pac-Man gets eaten by a ghost. Since reaching such a state is heavily penalized, a simple advice would be to avoid such states. Given a horizon $H$ and a state $s\in S$, the search would be restricted in order to satisfy the path formula $\varphi^H = \Box^{\leq H}(\neg loss)$ that encodes that safety constraint.

A more quantitative approach would compute for each action $a$ and over all policies the best probability $\eta_{H}(s,a)$ to satisfy $\varphi^H$ when the action $a$ is taken from $s$: 
$\eta_{H}(s,a)=\sup_{\sigma:\sigma(s)=a}\Pr_s^{\sigma}(s\models\varphi^H)$.
Then, the advice restricts Pac-Man to almost-optimal actions, \textit{i.e.}~decisions $a$ where $\eta_{H}(s,a)\ge t\times\max_{a'} \eta_{H}(s,a')$, where $t$ is a threshold in $[0,1]$.
Probabilistic model-checkers such as Storm can accept logical specifications in PCTL and compute the probability of path formul\ae{} $\varphi$. 
This approach is similar to probabilistic shielding~\cite{DBLP:conf/concur/0001KJSB20} where bad actions are pre-calculated and used to safely explore the search space during reinforcement learning.
A notable difference is that building such a shield requires one to construct the entire state-space of the MDP, whereas our approach performs its computations on-the-fly based on the current position alone.
Note that these computations are frequently performed on smaller models, for example in {\sc{Pac-Man}} we only need to consider a safety-relevant variant of the MDP where food pellets are ignored and that is restricted to states at distance at most $H$ from the current state. 
The practical interest of such advice for MCTS is detailed in~\cite{DBLP:conf/concur/Busatto-Gaston020}.

\section{Imitating expert policies}

In order to reach on-the-fly computing times low enough for real-time control,
we  
train  
policies, encoded as neural networks, 
to imitate an
expert policy.
This can take different forms depending on the expert policy $\sigma$.
In general, we define a function $f_{\sigma}:S\times A\to \R$ encoding the policy $\sigma$
so that from state $s$, the decision made by $\sigma$ is equivalent to choosing an action from $\argmax_{a\in A}(f_{\sigma}(s,a))$ uniformly at random.
Intuitively, $f_{\sigma}$ is a scoring function that rates how good every action is from the current state.
To learn a memoryless policy $\sigma$, this function can output the expected total reward under $\sigma$, or a heuristic score approximating it as returned by MCTS for example.
This framework can also be used to learn quantitative advice, e.g. by using $\eta_H(s,a)$ as a scoring function. In this case, the advice is seen as a (non-deterministic) expert policy to be imitated. This way, we suggest that a symbolic advice can also be imitated by a neural network that can then be used as a \emph{neural advice} in MCTS.

The plan is to teach 
a neural network 
the function $f_{\sigma}$ as an offline step,
and
use it to speed up the computations of decision-time planning.
Depending on  
the set of actions, we can either 
train a neural network that takes a state-action pair $(s,a)$ 
and outputs a single value $f_{\sigma}(s,a)$ or a neural network that takes a state $s$ 
and outputs a vector in $\R^{|A|}$ with
values for each available actions.

We address the following challenges:
encoding a state $s$ and its corresponding values $(f_{\sigma}(s,a))_{a\in A}$ so that it is easily processable by the neural network,
generating data for $f_{\sigma}$ representative of the state-space,
choosing an architecture for the neural network,
comparing the learnt policy and the expert policy $\sigma$.
\subsection{Training a neural network}
We divide our datasets in $5:2:3$ ratios to create distinct datasets for training, validation and testing of the neural networks.
We propose the use of convolutional neural networks which would take a state in the MDP as a \emph{tensor} with each channel of the tensor representing different features extracted from the state. In {\sc{Pac-Man}}, each tensor representing a state has $7$ channels to denote respectively the distribution of walls, food pills, position of Pac-Man, and for each direction, positions of the ghosts who are moving towards that direction. For example the channel representing the distribution of walls would be a matrix $w_{ij}$ of the size of the grid where $w_{ij} = 1$ if there is a wall at the co-ordinate $(i,j)$, and otherwise $w_{ij} = 0$.

We considered different approaches for normalization, either by globally scaling the values between $0$ and $1$ so that
$\min_{s,a}f(s,a)$ becomes $0$ and $\max_{s,a}f(s,a)$ becomes $1$ after normalization,
or scaling
locally so that for all state $s$,
$\min_{a}f(s,a)$ becomes $0$ and $\max_{a}f(s,a)$ becomes $1$. 
We argue that this local normalization is sufficient to learn the policy as it captures the ordering of the actions. Experimentally, local normalization performed better than global normalization. We also experimented with non-linear transformations~\cite{box1964analysis,yeo2000new} 
but they did not improve learning performances in our settings.
Our neural networks contain a 2D convolution layer with $3\times 3$ filters, a flattening layer, a few dense layers with the ReLU activation function and a final dense layer with the sigmoid activation function. Training is performed using ADAM optimizer~\cite{DBLP:journals/corr/KingmaB14} with mean squared error as loss function. To choose the optimal hyperparameters, \textit{e.g.}~the exact number of layers and their size or the number of filters, we use \emph{hyperparameter tuning}~\cite{bergstra2012random} in each setting. In particular, we relied on the Python library {\sc KerasTuner}~\cite{omalley2019kerastuner}.

\subsection{Formally sharp DAgger} 
Let us detail how to construct a set of data of the shape $(x,y)$, where $x\in S$ is the input of the neural network and $y\in\R^{|A|}$ is its output encoding $(f_{\sigma}(x,a))_{a\in A}$.
We argue for the use of formal methods
in order to answer:
how to get a representative set of input values $x$, and how to get good $y$ values for this set of input.

\paragraph{Perfect data}
Note that an expert policy generated by an exact method is ensured an expected payoff higher than any expert policy generated from a heuristic approach. In a sense, if one sees a heuristic approach as an approximation of the optimal policy, the data obtained from heuristic policies
can be seen as a noisy version of data that would otherwise be ``perfect'', \textit{i.e.}~pairs $(x,y)$ where $y$ is a vector encoding the decisions of a policy $\sigma$ that is optimal.

\paragraph{Representative set of inputs}
In order to generate a dataset to train on, a classical method is to pick states and actions uniformly at random within the state-space and to evaluate $f_{\sigma}$ on these inputs.
For example, one can consider Frozen Lake states obtained by placing the walls, the holes, the target and the robot at random empty positions.
However, a neural network trained from such a dataset may perform poorly for
states that play a key role in the expected payoff of a policy (\textit{i.e.}~states that represent crucial decisions), as such states may not be likely to be selected at random within the state-space.
The DAgger (Dataset Aggregation) algorithm, in contrast, offers a dataset generation method based on running simulations in order to get a more realistic view of the states frequently encountered in real plays.
While this approach can be part of the answer, it may not provide sufficiently many datapoints on the crucial decisions mentioned before, that may be few and far-between.

We propose an algorithm named \emph{sharp DAgger} that would detect these states, refine the training set and retrain the network. This is done by simulating the policy using the learnt neural network on the MDP and finding \emph{counter-examples} where the neural network is performing poorly by comparing the value given by the network and the value $f_{\sigma}(s)$ associated with the exact method. 
In Algorithm~\ref{alg:cap}, we present a method to train the neural network by an iterative process that generates new data for the training set. In the first iteration, we train a neural network $\NN_0$ from an initial training dataset $\mathsf{DATASET}$ and in later iterations, we add more interesting data-points in that set. 
Initially, one could either randomly generate a small amount of data 
or simulate the MDP by following a uniform policy. 
In iteration $i$, starting from an initial state $s_0$ in the MDP, we simulate a fixed number of paths until a given horizon $H$. We extract from these paths the states 
for which the current neural network $\NN_i$ trained from $\mathsf{DATASET}$ fails to predict the correct values. We add them to our dataset, then train the next iteration of the neural network.
The decision on when to stop the sharp DAgger loop is taken based on evaluations of
the quality of the neural network $\NN_i$ at each iteration $i$.

\begin{algorithm}[t]
	\caption{Sharp Dataset Aggregation (Sharp DAgger)}\label{alg:cap}
	\KwIn{
		A function $f_{\sigma}:S\to \R^{|A|}$ encoding an expert policy $\sigma$, $s_0\in S$, a metric $d$, $\epsilon\in\R$, a horizon $h\in \N$.}
	\KwOut{A policy $\sigma_i$ that imitates the policy $\sigma$}
	$\mathsf{DATASET} = $ initial dataset\;
	$\NN_0 = $ neural network trained using $\mathsf{DATASET}$\;
	$\sigma_0 = $ policy extracted from $\NN_0$\;
	\For{$0 \leq i\leq iters$}{
		$\FPaths_i$ = paths sim'd following $\sigma_i$ from $s_0$ for $h$ steps\;
		\For{state $s$ in paths $p\in \FPaths_i$}{
			\If{$d(\NN_i(s),f(s))\geq\epsilon$}{
				Add $(s,f(s))$ to $\mathsf{DATASET}$\;
			}
		}
		$\NN_i = $ neural network trained using $\mathsf{DATASET}$\;
		$\sigma_i = $ policy extracted from $\NN_i$\;
	}
	\Return $\sigma_i$;
\end{algorithm}

\subsection{Evaluating a learnt policy}
In order to evaluate the trained neural network, a traditional approach for machine learning can report on a loss function for a test dataset. Alternatively, one can measure the accuracy of the network by reporting how many times the resulting learnt policy has differed from the expert policy as a classifier. But this may not be sufficient to evaluate how the learnt policy is performing on the MDP. 
In Frozen Lake, consider a learned policy that returns the same action as the expert policy for all states in the MDP, except for one state where the learnt policy gives a bad action that leads to a hole.
Even though the learnt policy has an almost perfect accuracy, it would perform badly compared to the expert policy in real plays, and could lead to much worse rewards on expectation.

As such, we argue for the use of \emph{statistical model checking} to evaluate the expected reward of a (neural) policy.
In particular, we can use the approximate probabilistic model checking  method~\cite{herault2004approximate} where we simulate a set of paths following the expert policy on the one hand and the neural policy on the other, then compare their average rewards on these paths.

\begin{theorem}
	Suppose for MDP $M$, there exists $a<b$ such that $a\leq \Reward^h_{M}(p)\leq b$ for all paths $p$ in $M$. Let $\delta\in (0,1]$ and $\epsilon\in (0,b-a]$. Then for a policy $\sigma$, suppose we sample $n\geq \frac{(b-a)^2}{2\epsilon^2}\ln(\frac{2}{\delta})$ paths $p_1,p_2\ldots p_n$ independently at random from a state $s$ in the MDP~$M$ following the policy $\sigma$. Let $\overline{r}=\frac{1}{n}\sum_{i=1}^n\Reward^h_{M}(p_i)$. Then,
	$$\Pr_s^{\sigma}(|\overline{r}-\Val^h_{M}(s,\sigma)|\geq \epsilon) \leq \delta\,.$$
\end{theorem}
\begin{proof}
	We have $n$ independent identically distributed random variables $\Reward^h_{M}(p_i)$ with expected value $\Val^h_{M}(s,\sigma)$. Then, $\E_s^{\sigma}(\overline{r}) = \Val^h_{M}(s,\sigma)$. The Chernoff-Hoeffding inequality~\cite{Hoeffding} then yields
	$\Pr_s^{\sigma}(|\overline{r}-\Val^h_{M}(s,\sigma)|\geq \epsilon) \leq 2\exp\left(-\frac{2n\epsilon^2}{(b-a)^2}\right)\leq \delta\,.\qedhere$
\end{proof}
The above theorem gives a theoretical bound on the number of simulations needed to get a \emph{probably approximately correct} approximation of the real expected reward. In practice, we typically need fewer simulations to achieve a good approximation. For example, consider the Frozen Lake layout in \figurename~\ref{fig:FrozenLake}. 
Using exact methods  
we calculated the optimal expected reward 
to be $0.827$. In \figurename~\ref{fig:SMC}, for $n\in [1,100]$, we independently simulated $n$ paths using the optimal policy and plotted the estimated reward obtained from statistical model checking. We see that we get a good approximation of the real expected reward with under $100$ simulations.
\begin{figure}[h]
	\centering
	\includegraphics[width=0.89\columnwidth]{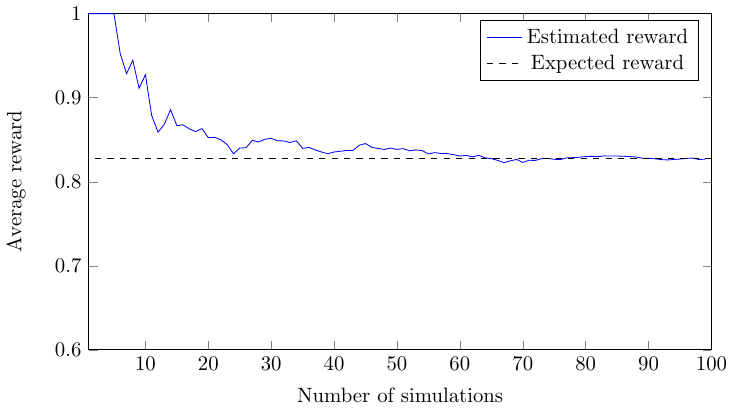}
	\caption{
		Statistical model checking for Frozen Lake
	}\label{fig:SMC}
\end{figure}

\section{Experimental results}
We ran experiments on the two MDPs previously introduced in Section~\ref{sec:MDP}.
Frozen Lake is an MDP that can be fully handled by model-checkers (using exact methods), and as such we use it to report on the benefits of using perfect data to train the surrogate policy.
Whereas, the {\sc Pac-Man} game provides more challenging MDPs to handle. There, we report on the performance of MCTS equipped with perfect or neural advice and on the performance of a surrogate policy trained on data obtained from MCTS. The sharp DAgger algorithm (Algorithm~\ref{alg:cap}) proves to be instrumental for learning efficiently in {\sc Pac-Man}.
The code is available at \cite{debraj_chakraborty_2023_7655528}.
\subsection{Frozen Lake}
For the game described in Example~\ref{ex:frozenLake}, we randomly generated layouts of size $10\text{x}10$ where we place walls at each cell in the border of the grid and with probability $0.1$ at each of the other cells. Then we place holes in remaining cells with probability $0.1$. Finally, we randomly place a target and an initial position in two of the remaining empty cells.
If the game is neither won nor lost within $1000$ steps, the game is considered a draw. 

\subsubsection{Expert policies}
Consider the state in the MDP where the robot is on the target position. We label this state with $target$. Using the model checker, we can compute the policies $\Opt(s) = \argmax_{\sigma}\Pr_{s}^{\sigma}(s\models \Diamond target)$ that maximize the probability to reach the target $t$ starting from state $s$. 
The practical policy that we are interested in should not only maximize the probability to reach the target but also minimize the expected number of steps needed to reach the target (in order to reach it before the horizon $H$ whenever possible). 
For a path $\rho$ in MC $M_{\sigma}$, we define $\len(\rho,target) = i$ if $\rho[i]$ is the target state and for all $j<i$, $\rho[j]$ is not the target set.
Using formal methods techniques, we can calculate a policy in
$$\argmin_{\sigma\in\Opt(s)}\E_s^{\sigma}(\len(\rho,target)\mid \rho\models \Diamond target))\,.$$
This policy can be shown to be optimal for 
total reward of any large enough horizon $H$.
We compared it with the policy generated from MCTS with horizon $H = 30$. From state $s$, a search tree is constructed 
for $40$ iterations. Thus, the search tree constructed by the MCTS algorithm contains up to $40$ nodes. In each iteration, 
when a new node is added to the search tree, $10$ samples are obtained by using a uniform policy to estimate the value of the node.

\subsubsection{Learnt policies}
Our training dataset contained $760k$ data-points which we used to imitate the expert policies. Hyperparameter tuning resulted in neural networks containing a 2D convolution layer with $6$ filters, a flattening layer and $2$ dense layers.
We randomly generated $1000$ layouts and ran $100$ games from each layout for $1000$ steps using both expert policies and the learnt policies. 
The average outcomes are reported in Figure~\ref{fig:FLPlot}.
Using Storm, we calculate the optimal expected win rate to be $93\%$ on average
in the generated layouts. This value denotes the probability to reach the target eventually, using the optimal policy. 
In practice, our statistical model checking approach requires fixing a finite horizon.
Figure~\ref{fig:FLPlot} confirms that 
horizon $1000$ is sufficient 
as the expert policy from Storm
still reaches a win rate of $93\%$. 
In comparison, our policy learnt from Storm had a win rate of $81\%$.
The expert policy calculated using MCTS is suboptimal and showed a win rate of $77\%$ while the policy learnt from it has a win rate of $69\%$. This highlights the benefits of using exact methods to get noise-free data. 
\begin{figure}[h]
	\centering
	\includegraphics[width=\columnwidth]{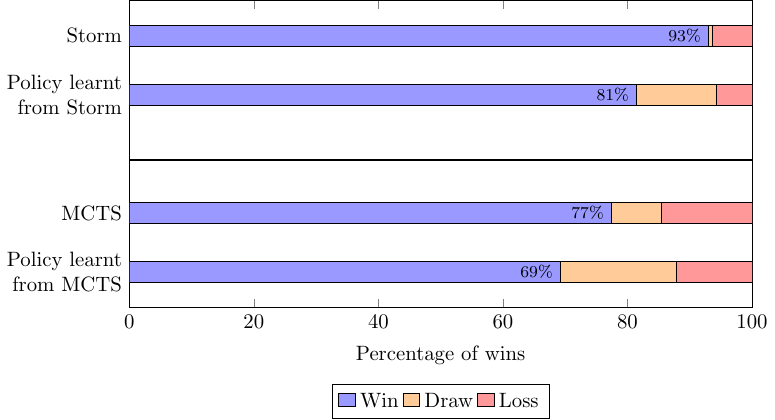}
	\caption{Perfect vs MCTS-based policies for Frozen Lake.
	}\label{fig:FLPlot}
\end{figure}

\subsection{Pac-Man}

We performed our experiments on the game {\sc{Pac-Man}} in a grid of size $9\times 21$ described in Figure~\ref{fig:Pacman}. 
In our experiments,
the ghosts always choose an action uniformly at random from the legal actions available. As explained in Example~\ref{ex:pacman}, we can view this as an MDP.
Moreover, if Pac-Man does not win (eats all food pills) or lose (makes contact with a ghost) within $300$ steps, we consider it a draw.

\subsubsection{Expert policies}
The state-space of the MDP is too large to apply directly to find the optimal policy. As a consequence, we decided to use Monte Carlo tree search with a receding horizon of $H = 10$. 
From state $s$, a search tree is constructed with a maximum depth of $H$ for $40$ iterations.\footnote{$40$ iterations was selected experimentally as a good compromise between achieving high expected rewards and minimising computation time.}
We combined MCTS with the notion of advice as used in \cite{DBLP:conf/concur/Busatto-Gaston020} in order to play Pac-Man.
In each iteration of the MCTS algorithm, when a new node is added, $20$ samples are obtained by using a uniform policy to estimate the value of the node among the paths that are safe i.e.~where Pac-Man is not eaten by a ghost. This optimistic estimation matches the notion of \emph{simulation advice} of~\cite{DBLP:conf/concur/Busatto-Gaston020}.
During the exploration of the search tree, we also restrict ourselves to actions $a$ that maximize the probability to stay safe for the next $8$ steps, \textit{i.e.}, actions $a$ such that
$\eta_8(s,a)=\max_{a'\in A}\eta_8(s,a')$ as defined in Section~\ref{sec:MCTSa}.
Since the online computation of the $\eta_8$ function is too expensive to be done at every node of the search tree, we only restrict the root node of the tree so as to ensure the safety of the immediate decisions.

We compare four different variants of MCTS in Figure~\ref{fig:pacmanPlot}: a version without this expert (safety) advice, one where it is used at the root node of the tree, one where a neural advice is trained to imitate the safety advice and is used at the root node, and finally one where the neural advice is used at every node in the tree. For reference, 
\cite{DBLP:conf/concur/Busatto-Gaston020} reports that human players win $44\%$ of the time on this grid.

\begin{figure}[t]
	\centering
	\includegraphics[width=\columnwidth]{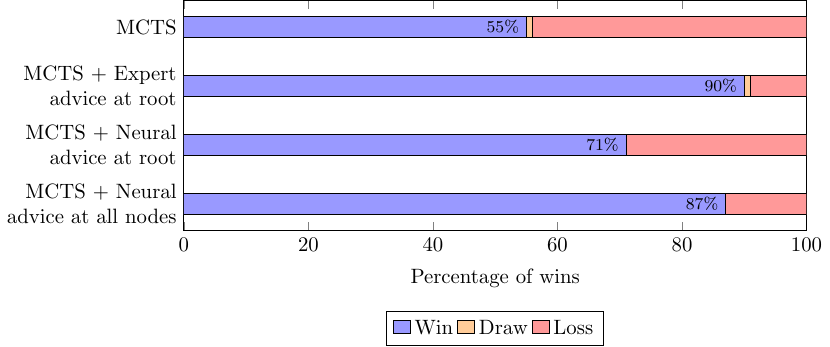}
	\caption{Different MCTS variants for Pac-Man.
	}\label{fig:pacmanPlot}
\end{figure}

\subsubsection{Neural advice}
To speed up the MCTS procedure we train a neural network to imitate the safety advice. 
We used Algorithm~\ref{alg:cap} to create a dataset. 
We use the \emph{$L_{\infty}$ metric} 
with precision value $\epsilon = 0.2$ to find new data-points during the aggregation. In other words, we add $(s,(\eta_{H}(s,a))_{a\in A})$ to the dataset at the $i^{th}$ iteration of sharp DAgger if
$\max_{a\in A}(|\eta_{H}(s,a)-\NN_i(s,a)|)>0.2$.
\begin{figure}[h]
	\centering
	\includegraphics[width=\columnwidth]{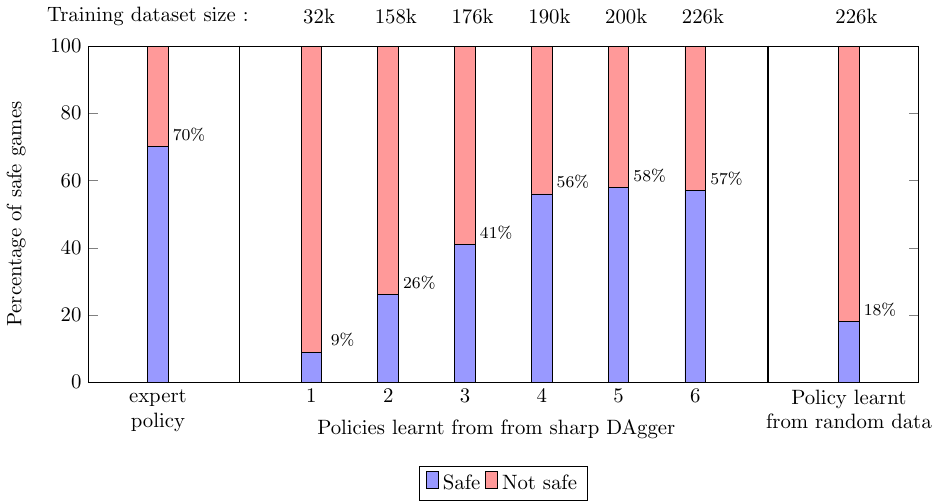}
	\caption{Sharp DAgger for {\sc Pac-Man} neural advice
	}\label{fig:safetyPlot}
\end{figure}
In each iteration, we simulate $4000$ games for $300$ steps to generate $\FPaths_i$.
We compare the safety status of the neural networks at each iteration of sharp DAgger in Figure~\ref{fig:safetyPlot}.
After $5$ iterations, we observe that Pac-Man stays safe (for $300$ steps) in $58\%$ of games when using the learnt policy instead of staying safe in $70\%$ of games
with the policy calculated from model checking.	
Hyperparameter tuning stabilized on neural networks using a 2D convolution layer with $6$ filters, a flattening layer and $4$ dense layers.
The entire training dataset generated from sharp DAgger contains $226k$ data-points. 
To check the effectiveness of our method of data aggregation, we compare our learnt policy with a policy trained on $226k$ randomly generated data-points.
This learnt policy performs worse and stays safe in only $18\%$ of games.
Dataset generation and training of the neural networks was performed in 36 hours with a cluster of 250 CPU cores, for a total of 9000 hours of computing time (at 2.9 GHz).

\subsubsection{Using the neural advice in MCTS}
To accommodate for the inherent noise in the output of the neural network $\NN$, we fix a threshold $t = 0.9$ and 
consider the advice that
allows almost-optimal actions with respect to~$t$,
\textit{i.e.}~the neural advice that restricts to 
actions $a$ such that $\NN(s,a)\geq 0.9\times \max_{a'\in A}NN(s,a')$.

We compare in Figure~\ref{fig:pacmanPlot} the performance of MCTS variants using expert or neural policies as advice. 
We ran each setup on $100$ games. 

The Python implementation of MCTS that we rely on was not designed to optimize the performance in terms of computing time. In our case, the MCTS algorithm without any selection advice uses $9$ seconds to decide on an action. Using the (formal methods based) expert advice at the root node of MCTS increases the time per decision by $8$ extra seconds. While the $9$ seconds spent in MCTS can be expected to be vastly lowered using code improvements,\footnote{MCTS and other simulation-based techniques are highly amenable to parallelism  \cite{Chaslot08}.} the model checking done by Storm is already optimized. By replacing the expert advice with a neural advice,
we can avoid this fixed cost of $8$ seconds per decision, as the network can be consulted in $3$~ms instead.
While the neural advice is not as good as the expert advice (it 
ensures safety in $71\%$ of games instead of $90\%$ when used identically at the root node of the MCTS tree), we can afford to use it on every node of the search tree to dynamically prune the search.
In this way, we can get an $87\%$ win-rate that is the best of both worlds: we approach the win-rate of the expert advice with the computing time of the bare-bones version of MCTS.
Since the neural advice requires much less computational power per call than the expert advice, using it would compensate the expensive computational cost of its training in the long run. In our case, we break even after 4 million calls (roughly 40k games of Pac-Man).

\subsubsection{Learning a surrogate policy}

We trained a surrogate neural network to imitate the expert policy defined previously as MCTS with a neural advice at every node, that reached an $87\%$ win-rate while keeping computing times as low as possible.
To generate the dataset, we use our sharp DAgger algorithm and simulate $4000$ games with horizon $300$ in each iteration. 
To evaluate how well our policies are performing, 
we compare the average number of wins obtained by following them in Figure~\ref{fig:mctsPlot}.
\begin{figure}[h]
	\centering
	\includegraphics[width=\columnwidth]{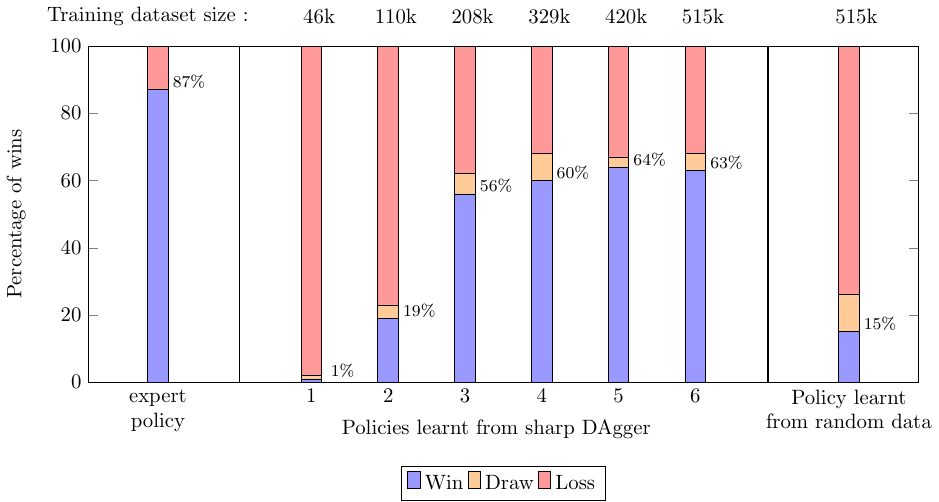}
	\caption{Sharp DAgger for {\sc Pac-Man} surrogate policies.
	}\label{fig:mctsPlot}
\end{figure}

After $5$ iterations, we  reach a policy with a win-rate of $64\%$, which is 
higher than the $55\%$ of the ``standard'' version of MCTS,
while having almost no need for 
online computing time as it is using a pre-trained neural network. 
Hyperparameter tuning stabilized on neural networks using a 2D convolution layer with $5$ filters, a flattening layer, $5$ dense layers.
Finally, the training dataset generated with sharp DAgger contains $515k$ data-points. In comparison, a policy learned from a randomly generated dataset
of size $515k$ is only able to win in $15\%$ of games, which confirms the importance of sharp DAgger in this setting.

\section{Conclusion}
In this work, we have proposed a framework to combine formal methods with MCTS and deep learning to obtain a scalable way of synthesising policies with both good performance and low latency.

From our experiments, 
we conclude that 
formal methods can provide good policies and useful advice for MCTS, albeit at a high computational cost. Training a neural network to play the role of the advice
allows one to obtain the best of both worlds: the performance boost of the advice but without its computational cost. Particularly, neural advice compensates for its expensive computational training cost in the long run since it requires less computational power per call than expert advice. Using a sharp dataset-aggregation procedure is instrumental in reaching satisfactory rewards in practice because of the reliance of deep-learning techniques on the accumulation of huge amounts of data. Finally, while
the best policy that we obtained for {\sc Pac-Man} is based on MCTS, its surrogate neural-network policy is able to play relatively well while making near instantaneous decisions.



\begin{acks}
Computational resources have been provided by the C\'{E}CI, funded by the F.R.S.-FNRS under Grant No. 2.5020.11 and by the Walloon Region. This work was supported by the ARC ``Non-Zero Sum Game Graphs'' project (Fédération Wallonie-Bruxelles), the EOS ``Verilearn'' project (F.R.S.-FNRS \& FWO), and the FWO ``SAILor'' project (G030020N).
\end{acks}



\bibliographystyle{ACM-Reference-Format}
\bibliography{references}


\end{document}